%% file: main.tex
\documentclass[letterpaper, 10 pt, conference]{ieeeconf}  

\IEEEoverridecommandlockouts                              

\usepackage{amsmath} 

\usepackage[symbol]{footmisc}

\usepackage{adjustbox}
\usepackage{multirow}
\usepackage{array}
\usepackage{booktabs}

\usepackage{makecell}

\usepackage{algorithm}
\usepackage{algorithmic}

\usepackage{amsmath, amssymb}

\usepackage{amsthm}

\usepackage{enumerate}
\usepackage{mathtools}

\newtheorem{theorem}{Theorem}[section]
\newtheorem{lemma}[theorem]{Lemma}
\newtheorem{proposition}[theorem]{Proposition}

\newtheorem{problem}[theorem]{Problem}

\usepackage{comment}
\usepackage{tikz}

\newcommand{\norm}[1]{\left\lVert#1\right\rVert}

\newcommand{\boldparagraph}[1]{\vspace{0.5em}\noindent{\bf #1}}

\newcommand{\cross}[1][1pt]{\ooalign{%
  \rule[1ex]{1ex}{#1}\cr
  \hss\rule{#1}{.7em}\hss\cr}}%

\title{\LARGE \bf
Ternary-Type Opacity and Hybrid Odometry for RGB NeRF-SLAM
}

\author{Junru Lin$^{1,2,3}$ \quad Asen Nachkov$^{1}$ \quad Songyou Peng$^{3}$ \quad Luc Van Gool$^{1,3}$ \quad Danda Pani Paudel$^{1,3}$
\thanks{$^{1}${INSAIT, Sofia University, Bulgaria}}%
\thanks{$^{2}${University of Toronto, Canada}}%
\thanks{$^{3}${ETH Zurich, Switzerland}}%
}

\usepackage{graphicx}
\usepackage{caption} 

\begin{document}


\maketitle
\thispagestyle{empty}
\pagestyle{empty}

\input{sec/0_abstract}    
\input{sec/1_intro}
\input{sec/2_related_work}
\input{sec/3_1_background}
\input{sec/3_2_ternary-type_opacity}

\input{sec/3_3_hybrid_odometry}

\input{sec/4_experiments}
\input{sec/5_conclusion}





{
    \bibliographystyle{IEEEtran}
    \bibliography{main}
}

\end{document}

%% file: sec/0_abstract.tex
\begin{abstract}

In this work, we address the challenge of deploying Neural Radiance Field (NeRFs) in Simultaneous Localization and Mapping (SLAM) under the condition of lacking depth information, relying solely on RGB inputs. The key to unlocking the full potential of NeRF in such a challenging context lies in the integration of real-world priors. A crucial prior we explore is the binary opacity prior of 3D space with opaque objects. To effectively incorporate this prior into the NeRF framework, we introduce a ternary-type opacity (TT) model instead, which categorizes points on a ray intersecting a surface into three regions: before, on, and behind the surface. This enables a more accurate rendering of depth, subsequently improving the performance of image warping techniques. Therefore, we further propose a novel hybrid odometry (HO) scheme that merges bundle adjustment and warping-based localization. Our integrated approach of TT and HO achieves state-of-the-art performance on synthetic and real-world datasets, in terms of both speed and accuracy. This breakthrough underscores the potential of NeRF-SLAM in navigating complex environments with high fidelity.

\end{abstract}

%% file: sec/1_intro.tex
\section{Introduction}
\label{sec:intro}

The advent of Neural Radiance Fields (NeRFs) \cite{mildenhall2021nerf} has marked a paradigm shift in 3D scene reconstruction methodologies. As a result, several NeRF-based Simultaneous Localization and Mapping (SLAM) frameworks, referred to as NeRF-SLAM, have been developed recently~\cite{sucar2021imap, zhu2022nice, sandstrom2023point, wang2023bad, zhu2023nicer,rosinol2022nerf,li2023dense,zhang2023hi,zhang2023go}. Although existing research has significantly advanced NeRF-SLAM with RGB-D inputs, the necessity for RGB-D sensors limits their widespread adoption. Consequently, there is growing interest in RGB-only NeRF-SLAM. However, realizing the full potential of RGB NeRF-SLAM faces numerous challenges.

\begin{figure}[t]
    \centering
    \includegraphics[width=\linewidth]{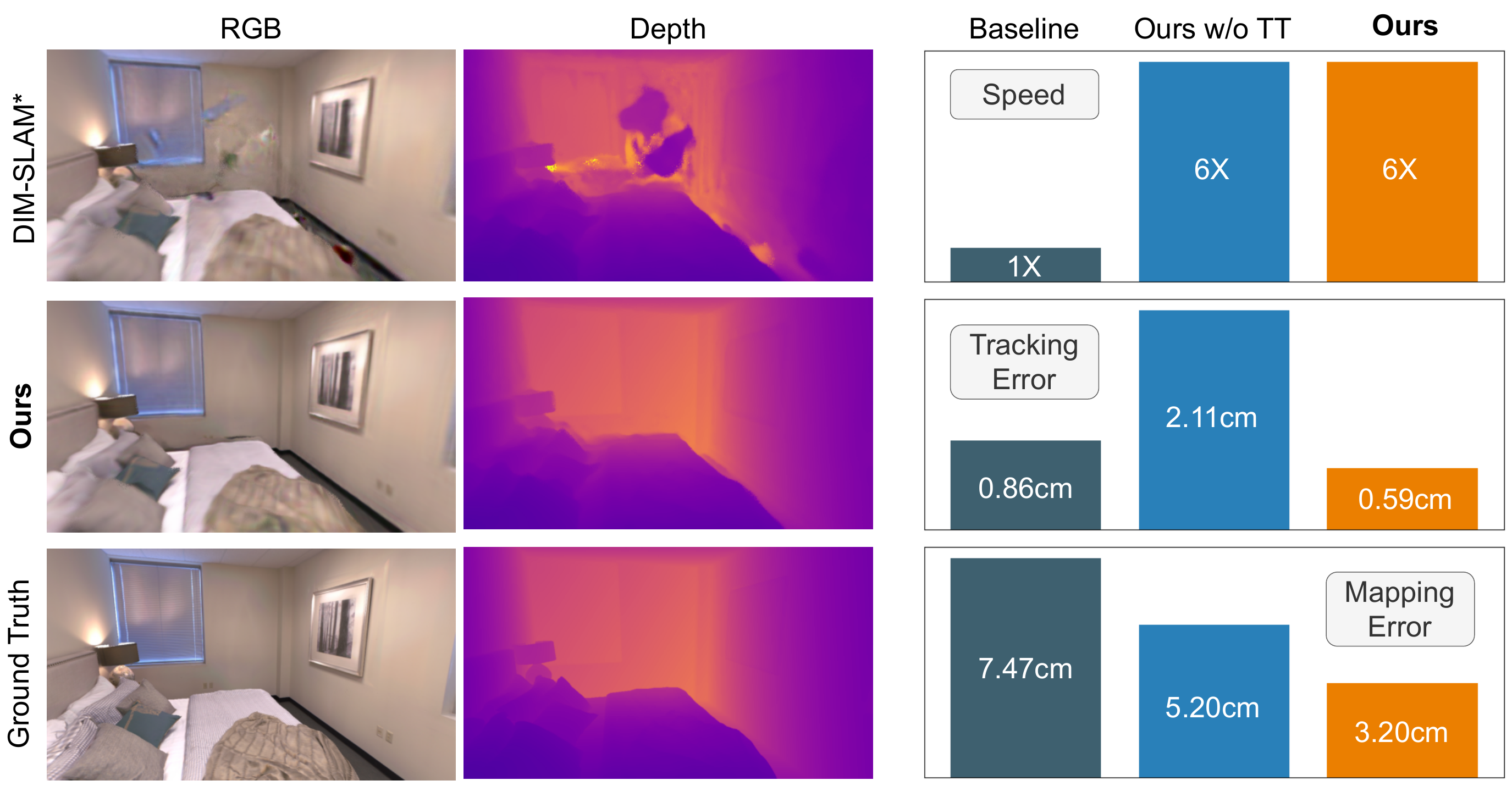}
    \caption{\textbf{Qualitative and Quantitative Results}. On the left, we show the rendered RGB and depth from a random pose after training the whole sequence of Replica~\cite{straub2019replica} \texttt{room-1}. On the right, we show the speed, tracking error, and mapping error on Replica \texttt{office-0}. DIM-SLAM$^*$ refers to our re-implementation for DIM-SLAM~\cite{li2023dense}.}
    \label{fig:teaser}
\end{figure}

Three noteworthy attempts for RGB-only NeRF-SLAM have been made in the literature: NeRF-SLAM~\cite{rosinol2022nerf}, DIM-SLAM~\cite{li2023dense}, and NICER-SLAM~\cite{zhu2023nicer}. Among them, NeRF-SLAM uses both pretrained (or separately optimized)  poses and depth networks, and NICER-SLAM uses pretrained depth networks, making the underlying problems associated with the task at hand obscure. Our focus is on the core issues within RGB NeRF-SLAM. DIM-SLAM avoids pretraining and shows promise. Unfortunately though, it suffers from instability and high computational requirements. We aim to refine DIM-SLAM's approach, and we identify that the speed and performance can be improved by resolving two major factors: (i) the failure to satisfy the binary opacity prior, and (ii) frame-wise volumetric rendering (VR).

\begin{figure*}[!t]
    \centering
    \includegraphics[width=\linewidth]{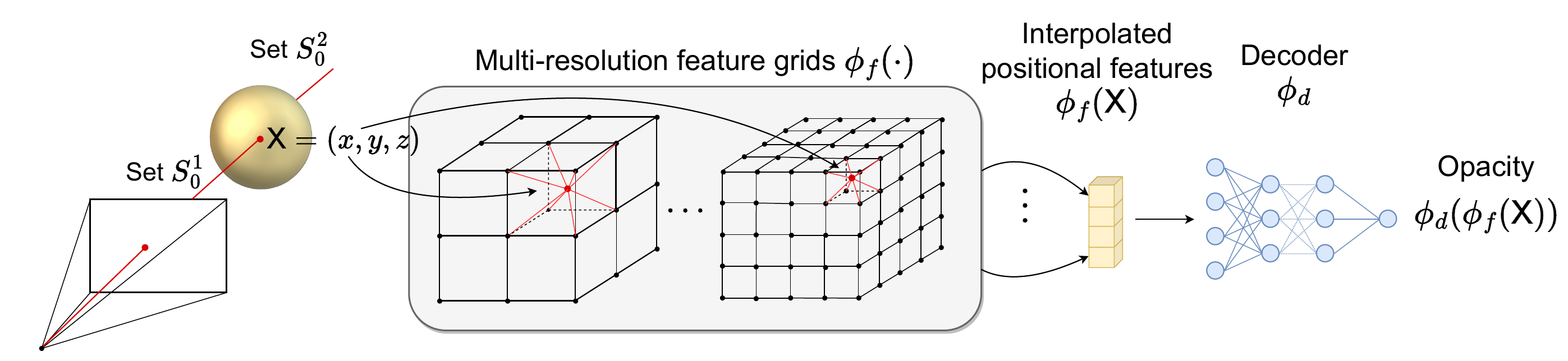}
    \caption{\textbf{Inferring the Opacity of a 3D Point}. We utilize a set of multi-resolution feature grids which we interpolate at the desired 3D point. The collected features are passed to a neural network to predict the color and the opacity. Only the opacity is shown for clarity.}
    \label{fig:schematic}
\end{figure*}

On the one hand, we aim to improve the scene representation by addressing the lack of observed binary opacity, which is expected in rigid 3D scenes with opaque surfaces. Unlike the RGB-D case, where rendered opacity is directly supervised using depths, the RGB-only case exhibits a widely spread distribution with almost no high opacity values, as illustrated in Fig.~\ref{fig:ternary_opacity} (blue part), which is undesirable. This motivates us to introduce a binary opacity prior into the RGB NeRF-SLAM pipeline. However, our analysis of the VR function reveals that integrating this binary opacity prior is challenging and may not be necessary. Instead, a ternary-type opacity (TT) is better suited for meaningful optimization. Intuitively, during VR, any point behind another opaque point does not contribute to the rendered image and can have any opacity and radiance. Their spread, however, turns out to be hindering during the binary opacity prior integration process. Therefore, we adopt a ternary-type model as illustrated in Fig.~\ref{fig:ternary_opacity} (orange part), which leads to the desired opacity and weight distribution along a ray (Fig.~\ref{fig:along_ray}) and thus improve performance.

On the other hand, we enhance camera tracking by avoiding the frame-wise VR. Instead, we propose to use a hybrid odometry (HO) technique that relies on image warping for coarse camera localization. The finer adjustments of the camera odometry are performed during the bundle adjustment (BA) step, which is conducted jointly during the mapping process. The proposed HO technique improves the speed up to an order of magnitude. This speed results from the fact that the pose initialization required for BA can be obtained simply by minimizing the photometric error between the current image and the warped previous images using the rendered depths (of the so-far reconstructed implicit map), similar to the existing works~\cite{engel2015large, engel2017direct}.

To summarize, our contributions are: (1) the introduction of a method to effectively leverage binary opacity prior for 3D surfaces through ternary-type modeling within the RGB-only NeRF-SLAM framework; (2) detailed theoretical insights into the impact of VR on mapping opaque 3D surfaces by distinguishing relevant from irrelevant components; (3) the proposal of HO and demonstration that the proposed TT prior and HO complement each other, achieving state-of-the-art results on benchmark datasets in terms of both speed and accuracy; and (4) a call to fully realize the potential of RGB-only NeRF-SLAM in future work.

%% file: sec/2_related_work.tex
\section{Related Work}
\label{sec:related_work}

\boldparagraph{Visual SLAM.} Simultaneous localization and mapping from 2D images is referred to as Visual SLAM \cite{fernandez2012simultaneous}. Feature-based algorithms estimate and track sparse or semi-dense keypoints across the image sequence \cite{engel2014lsd}, while direct methods estimate camera poses by minimizing a photometric error on the reconstruction \cite{macario2022comprehensive, engel2017direct}. The common stages include initialization, state prediction, tracking, and correction \cite{davison2007monoslam, vincke2010design, vincke2012efficient}. Different threads may be utilized for mapping and tracking\cite{newcombe2011dtam, klein2007parallel}, and loop closure detection can also be incorporated \cite{mur2015orb, mur2017orb}.

 \boldparagraph{Neural Implicit Representations.} Recently, NeRFs have emerged as an efficient method for high-quality scene mapping using neural networks \cite{mildenhall2021nerf}. Various improvements have been made to reduce the requirements of accurate and/or known camera poses \cite{wang2021nerf, lin2021barf, truong2023sparf}, to improve training speed \cite{muller2022instant}, and to increase the scale of the scene \cite{tancik2022block} and the rendering quality \cite{reiser2024binary, uy2023nerf}. 
 Among all related works, \cite{reiser2024binary} is the one that most relates to our method in binary opacity modelling. It encourages the density to converge towards surfaces using a discrete opacity grid representation and employs binary entropy loss for density. Our work incorporate the binary opacity in a more continuous manner, without discretizing the density values.

\boldparagraph{SLAM with Neural Implicit Representations.} 
Initial research on integrating NeRFs with traditional SLAM systems uses depth data for training \cite{johari2023eslam, lisus2023towards, zhang2023go}. iMAP \cite{sucar2021imap} uses a multilayer perceptron (MLP) to represent scene geometry and employs a keyframe graph to track frames, while NICE-SLAM \cite{zhu2022nice} and Point-SLAM \cite{sandstrom2023point} capture complex scenes and improve training efficiency using multiple MLPs and a sparse set of neural points, respectively. 

In contrast, methods like \cite{zhu2023nicer, rosinol2022nerf, li2023dense, chung2023orbeez} do not rely on explicit depth data. NeRF-SLAM uses a pretrained Droid-SLAM \cite{teed2021droid} to get estimated poses and depth directly for the tracking frontend. NICER-SLAM \cite{zhu2023nicer} uses a pretrained depth predictor and incorporates geometric cues based on RGB warping, optical flow, and surface normals. Differently, DIM-SLAM \cite{li2023dense} does not require any pretraining and relies on a photometric loss for geometry consistency. However, due to the per-frame BA, DIM-SLAM suffers from slow training speed. Our work advances DIM-SLAM, enhancing both the speed and performance of the SLAM system.

%% file: sec/3_1_background.tex
\section{Background}

In our SLAM system, the scene is represented using two neural functions for the implicit representation of opacity and the radiance of the scene, namely $\phi_{o}(\cdot )$ and $\phi_{c}(\cdot )$, respectively. When optimizing the radiance and opacity functions, we compare image frames with rendered images obtained by VR, which involves $\phi_{o}(\cdot)$ and $\phi_{c}(\cdot)$ simultaneously. Let $\mathcal{S}_l=\{\mathsf{X}_i\}_{i=1}^n$ be an ordered set of 3D points along a ray $l$ (emanating out and in front of the camera) that passes through 2d-pixel $\mathsf{x}$. The rendered RGB value at $\mathsf{x}$ is,
\begin{equation}
    \mathbf{I} (\mathsf{x}) = \sum_{\mathsf{X}_i\in{\mathcal{S}_l}}{w_i\phi_c(\mathsf{X}_{i})},
\end{equation}
and the corresponding depth is rendered similarly,
\begin{equation}
    D (\mathsf{x}) = \sum_{\mathsf{X}_i\in{\mathcal{S}_l}}{w_i D_i},
\end{equation}
where $D_i$ is the depth of point $\mathsf{X}_i$, and the weights $w_i$ are:
\begin{equation}
\label{eq:weights_orriginal}
    w_i = \phi_o(\mathsf{X}_i)\prod_{j=0}^{i-1}(1-\phi_o(\mathsf{X}_j)).
\end{equation}

We are interested in incorporating the prior for the opaque surfaces into the NeRF-SLAM system. Following the standard practice in NeRF \cite{mildenhall2021nerf}, we use sigmoid functions to model the output activations of both $\phi_{o}(\cdot)$ and $\phi_{c}(\cdot)$. We are now ready to introduce our problem.

\begin{figure}[!t]
    \centering
    \vspace{3pt}
    \includegraphics[width=\linewidth]{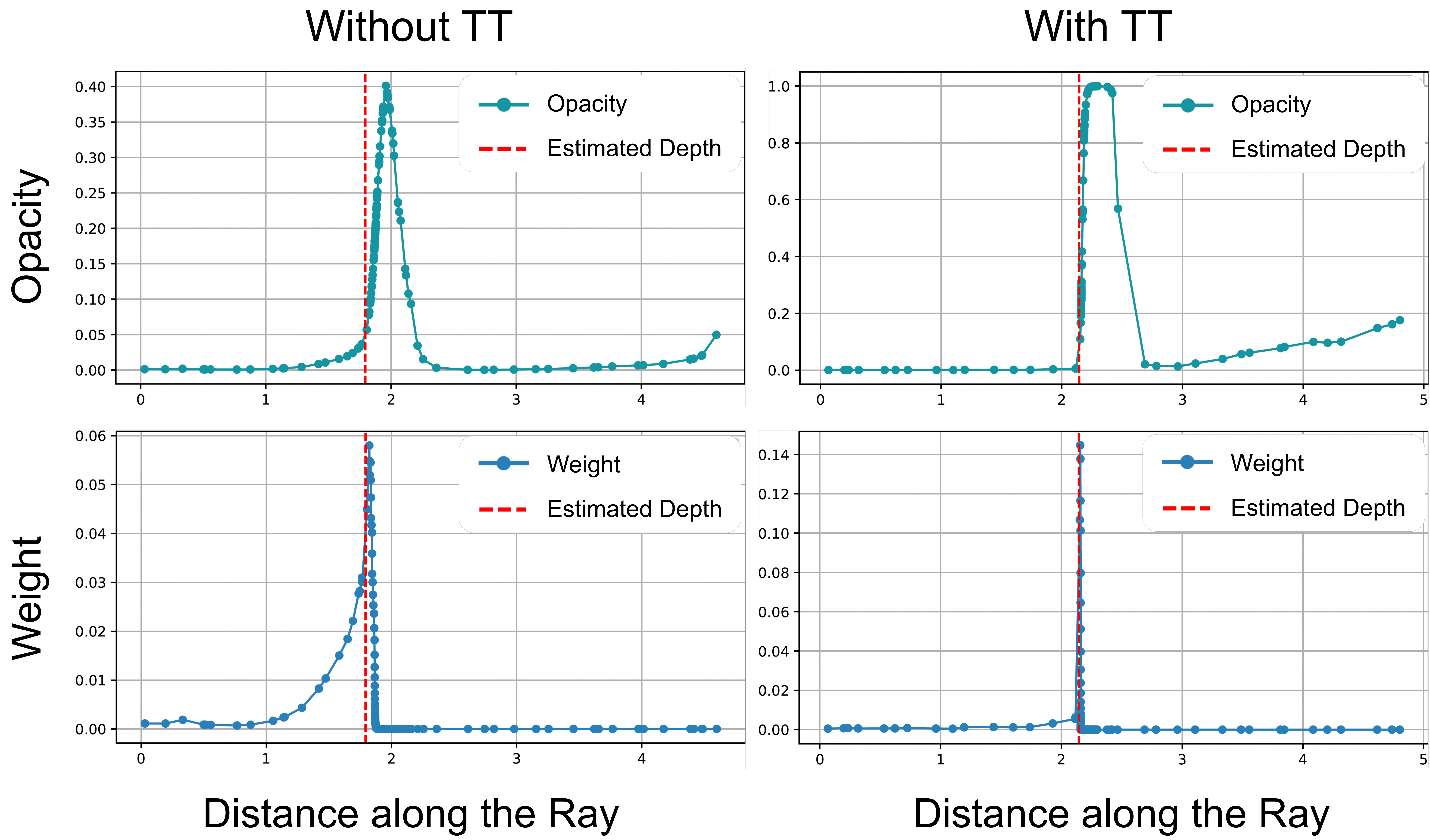}
    \caption{\textbf{Opacity and Weights along a Ray}. With the ternary-type opacity (TT), the weights along a randomly sampled ray are more concentrated near the depth with a higher peak. The dots on the curves represent the sampled points on the ray. Data obtained from Replica~\cite{straub2019replica} \texttt{office-0} at the end of the training.}
    \label{fig:along_ray}
\end{figure}

%% file: sec/3_2_ternary-type_opacity.tex
\section{Ternary-Type Opacity}
\label{sec:ttSec}

\begin{figure*}[t]
    \centering
    \vspace{3pt}
    \includegraphics[width=\linewidth]{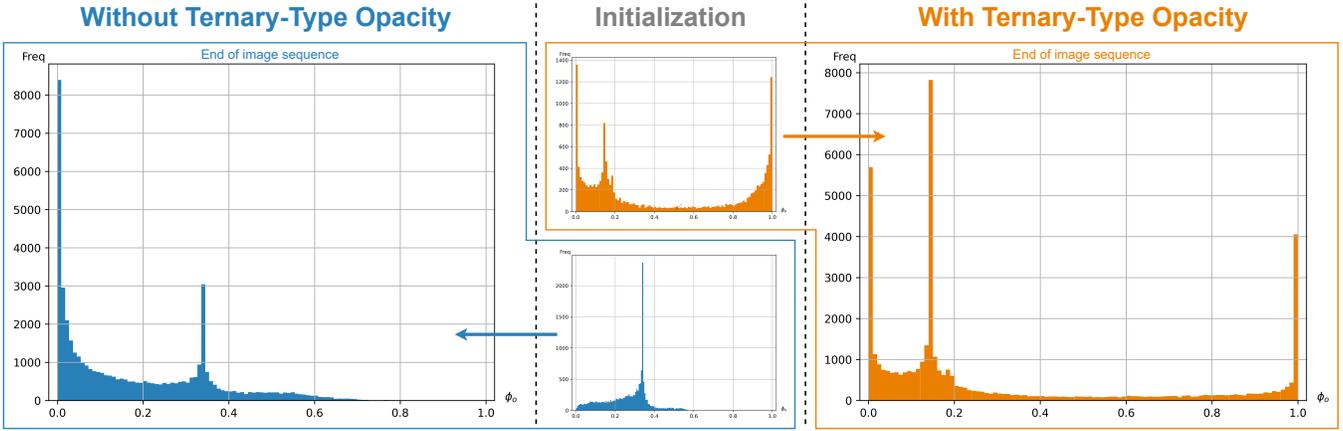}
    \caption{\textbf{Illustration of Ternary Opacity}. Randomly sampled 3D points have their features extracted through interpolation across multi-resolution feature grids. These features are input to a neural network to predict color and opacity. The resulting opacity histograms, shown above, demonstrate our method's effectiveness in generating the desired ternary-type opacity.}
    \label{fig:ternary_opacity}
\end{figure*}

In this section, we provide the insight of TT.

\begin{problem}\label{pb:binaryWeights}
Recover the radiance field by minimizing the image reconstruction error between the ground-truth (GT) color $c$ and rendered color, under an additional binary constraint on the weights, in the following form,
\begin{align}
    \label{eq:originalProblem}
    \begin{split}
    \underset{\phi, w_i}{\text{minimize}}\quad&\norm{ c - \sum_{\mathsf{X}_i\in{\mathcal{S}_l}}{w_i\phi_c(\mathsf{X}_{i})}}, \\
    \text{subject to}\quad &w_i\in \{0,1\}.
    \end{split}
\end{align}
\end{problem}
Correspondingly, Eq.~\eqref{eq:weights_orriginal} will have the constraint:

\begin{equation}
    w_i = \phi_o(\mathsf{X}_i)\prod_{j=0}^{i-1}(1-\phi_o(\mathsf{X}_j)) \in \{0,1\}.
\end{equation}

\begin{lemma} \label{lm:conditions}
The desired weight constraints $w_i\in\{0,1\}$  for the totally ordered set $\mathcal{S}_l$ can be achieved if and only if at least one of the following statements is true,
\begin{align*}
    \text{i.} \quad \phi_o(\mathsf{X}_i)&=0.\\
    \text{ii.} \quad \phi_o(\mathsf{X}_j)&=1 \;\; \text{for some} \;\; \mathsf{X}_j \in \mathcal{S}_l \;\; \text{with} \;\; j<i.\\
    \text{iii.} \quad \phi_o(\mathsf{X}_i)&=1 \;\; \text{and} \;\; \phi_o(\mathsf{X}_j)=0 \;\; \text{for all} \;\; j<i.
\end{align*}
\end{lemma}

\begin{proof}
Notice that $\phi_o(\mathsf{X})\in[0,1]$, which is a bounded interval.
For forward case,
${w_i = \phi_o(\mathsf{X}_i)\prod_{j=0}^{i-1}(1-\phi_o(\mathsf{X}_j))}$ results in $w_i\in\{0,1\}$ with cases (i), (ii), and (iii), which can be verified by computations leading to $w_i=0$ for (i), $w_i=0$ for (ii), and $w_i=1$ of (iii).

In case (i), all cases with $\phi_o(\mathsf{X}_i)=0$ are already covered. On the other hand, when $\phi_o(\mathsf{X}_i)<1$, the weights become $w_i<1$, as $\prod_{j=0}^{i-1}(1-\phi_o(\mathsf{X}_j)) \leq 1$. This leads to the only $w_i=0$ choice to be taken, with $\prod_{j=0}^{i-1}(1-\phi_o(\mathsf{X}_j)) =0$ being possible only if case (ii) is satisfied. Similarly, when $\phi_o(\mathsf{X}_i)=1$, $w_i\in\{0,1\}$ is possible only with $\prod_{j=0}^{i-1}(1-\phi_o(\mathsf{X}_j))\in\{0,1\}$. Under this condition, $\prod_{j=0}^{i-1}(1-\phi_o(\mathsf{X}_j))=0$ leads to case (ii), whereas $\prod_{j=0}^{i-1}(1-\phi_o(\mathsf{X}_j))=1$ leads to the case (iii). We exhausted all possibilities for $w_i\in \{0,1\}$, all of them leading to the listed three statements.
\end{proof}

\begin{lemma}\label{lm:optimal}
    For any given optimal solution of the Problem~\ref{pb:binaryWeights}, there exist nonenumerable sets of opacity measures ${\mathcal{S}_o^*  = \{\phi_o (\mathsf{X}_i)|\mathsf{X}_i\in\mathcal{S}_l\}}$ leading to the same outcome.
\end{lemma}

\begin{proof}
    The proof relies on the fact that some variables involved, specifically some $\phi_o (\mathsf{X}_i)$, influence neither the objective function nor the constraints. For the sake of simplicity and sufficiency, we provide the proof from the point of view of an additional point having freedom to have nonenumerable possibilities, without affecting the discussed objective and the constraints. 
    In any general setting, any $w_i\in\{0,1\}$ arising by satisfying any of the statements of Lemma~\ref{lm:conditions}   may have arisen with a different setting of  $\phi_o (\mathsf{X}_i)$. These are exactly the possibilities not covered by the statements in Lemma~\ref{lm:conditions}. Now it is straightforward to see that the possibilities not covered by the statements in Lemma~\ref{lm:conditions} are indeed nonenumerable. This leads to the nonenumerable sets of opacity measures ${\mathcal{S}_o^*  = \{\phi_o (\mathsf{X}_i)|\mathsf{X}_i\in\mathcal{S}_l}\}$ with the same outcome.  
\end{proof}

\begin{lemma}\label{lm:binary}
     Any set ${\mathcal{S}_o  = \{\phi_o (\mathsf{X}_i)\in \{0,1\}|\mathsf{X}_i\in\mathcal{S}_l}\}$ that minimizes the objective of Eq.~\eqref{eq:originalProblem}, is also a valid solution to the Problem~\ref{pb:binaryWeights}.
\end{lemma}

\begin{proof}
The proof is avoided as it is straightforward.
\end{proof}

\begin{theorem} [Relevant Binary-type Opacity]
\label{th:releventBinary}
For an ordered set $\mathcal{S}_o$ and ordered partition $\mathcal{S}_o = \{\mathcal{S}_o^1,\mathcal{S}_o^2\}$, let $\mathcal{S}_o^1=\{\mathsf{X}_k\}_{k=0}^i$. If $\mathcal{S}_o^1$ minimizes the objective of Eq.~\eqref{eq:originalProblem} with $\phi_o(\mathsf{X}) \in \{0,1\}$ for all $\mathsf{X}\in\mathcal{S}_o^1$, except for $\mathsf{X}_i$ when $\phi_o(\mathsf{X})=0$, then it is also a valid solution to Problem~\ref{pb:binaryWeights}.
\end{theorem}

\begin{proof}
The proof follows directly from the definitions and is therefore omitted for brevity.
\end{proof}

Note that in Theorem~\ref{th:releventBinary}, the irrelevant set $\mathcal{S}_o^2$ impacts neither the objective nor the constraint of the Problem~\ref{pb:binaryWeights}. Optimizing $\mathcal{S}_o^2$ is thus unnecessary. Yet, we have no knowledge of partition ${\mathcal{S}_o = \{\mathcal{S}_o^1,\mathcal{S}_o^2\}}$ beforehand. Additionally, some members of $\mathcal{S}_o^2$ might affect other pixels $c_p$. In either case, an eventual division between relevant and irrelevant sets with the very same properties (except the ordered subsets) is still achievable. Thus, we proceed with a single pixel-based formulation, involving a set $\mathcal{S}_l$ over the corresponding ray, maintaining generality for 3D space sets. 

In our setting, $\phi_o(\mathsf{X})$ can be initialized to any value for any chosen $\mathsf{X}$. Therefore, we propose to initialize with a straightforward fixed value of all $\mathsf{X}\in\mathsf{S}_l$. However, we also find that some trivial initializations lead to undesired outcomes during the iterative optimization.   

\begin{proposition}
Under the gray-world assumption, the initializations $\phi_{o}(\mathsf{X}) = 0, \forall \mathsf{X}\in\mathcal{S}_l$ or ${\phi_{o}(\mathsf{X}) = 1,}{ \forall \mathsf{X}\in\mathcal{S}_l}$  
 are sub-optimal for iterative optimization of Problem~\ref{pb:binaryWeights}.
\end{proposition}

\begin{proof}
When
$\phi_{o}(\mathsf{X}) = 0, \forall \mathsf{X}\in\mathcal{S}_l$ or ${\phi_{o}(\mathsf{X}) = 1,}{ \forall \mathsf{X}\in\mathcal{S}_l}$, this leads to $w_i=0$ in all cases. Hence, the term   
$\sum_{\mathsf{X}_i\in{\mathcal{S}_l}}{w_i\phi_c(\mathsf{X}_{i})} = 0$ in Equation~\eqref{eq:originalProblem} in the original Problem~\ref{pb:binaryWeights}, and both the initializations separately result in zero (also known as black) images. Under the gray world assumption such initialization is the farthest possible one. Let $[0,C_{max}]$ be the range of the color $c$. The gray-world assumption sets the expected value of $c$ to $C_{max}/2$. However, both initializations lead to $c=0$, which is the farthest value from $C_{max}/2$, which makes such initialization sub-optimal under the gray-world assumption. In other words, $C_{max}/2$ would have been a better initialization for the made assumption, and any initialization between $0$ (representing the zero image) and $C_{max}/2$, would have been better than the zero image.
\end{proof}

We resort to the map initialization step of the SLAM to initialize $\phi_o( \cdot )$. We first present our decomposition of $\phi_o( \cdot )$, which is illustrated in Fig. \ref{fig:schematic}.
To obtain the opacity of a 3D point $\mathsf{X}$, we embed it into a feature vector $\phi_f{(\mathsf{X})}$, and then decode the feature vector into opacity using decoder $\phi_d(\cdot )$ such that ${\phi_o(\mathsf{X})=\phi_d(\phi_f{(\mathsf{X})})}$. During the map initialization, we learn and subsequently freeze the decoder $\phi_d(\cdot )$, with only $\phi_f( \cdot )$ being optimized in later tracking and mapping stages. In the latter stage, as we initialize all $\phi_f(\cdot)$ with a constant value $\eta$, for any point in the space that has not yet been optimized, it always has $\phi_f(\cdot)$ being $\eta$, and thus also keep the opacity to the initial value, denoted as $o_{init}$. Now, we are interested in addressing the following problem. 
\begin{problem} \label{pb:problem2}
Optimize the objective function of Problem~\ref{pb:binaryWeights} such that $\phi_o(\mathsf{X})\in\{0,1\}$ for all $\mathsf{X}\in\mathcal{S}_o^1$, while fixing $\phi_o(\mathsf{X})=o_{init}$ for all the irrelevant part $\mathsf{X}\in\mathcal{S}_o^2$. 
\end{problem}

\begin{theorem}
    The solution to Problem~\ref{pb:problem2} will optimize Problem~\ref{pb:binaryWeights} without violating the constraints and with the optimal transport for $|\mathcal{S}_o^2|>>|\mathcal{S}_o^1|$. 
\end{theorem}

\begin{proof}
The first part of the proof concerns establishing the relationship between the solution to Problem~\ref{pb:problem2} to that of Problem~\ref{pb:binaryWeights}. In particular, the non-violation of the constraint is addressed. Note that Problem~\ref{pb:problem2} imposes the constraint  $\phi_o(\mathsf{X})\in\{0,1\}$ for all $\mathsf{X}\in\mathcal{S}_o^1$. The use of these constraints is in fact motivated by the results of Theorem~\ref{th:releventBinary}. Since Theorem~\ref{th:releventBinary} assures the validity of the mentioned constraints to Problem~\ref{pb:binaryWeights}, this part of the proof follows as a direct Corollary from the same. 

The second part of the proof makes use of the ordered partition proposed in  Theorem~\ref{th:releventBinary}. It has been shown that the role of $\mathcal{S}_o^2$ is irrelevant regarding making any difference to Problem~\ref{pb:binaryWeights}. It is important to notice that $\mathcal{S}_o^2$  is a part of the variables being optimized. However, $\mathsf{X}\in\mathcal{S}_o^2$ do not contribute to the original problem in any form. When the suggestion of Problem~\ref{pb:problem2} on fixing $\phi_o(\mathsf{X}) = o_{init}$ for all irrelevant part $\mathsf{X} \in\mathcal{S}_o^2$ is considered, it lowers down the fraction of variables to be optimized, and in extreme cases to zero (or close to zero). When the distributions of $\phi_o(\mathsf{X})$  before and after the optimization are compared, it becomes clear that fixing $\phi_o(\mathsf{X}) = o_{init}$ leads to the optimal solution with the optimal transport, for  ${|\mathcal{S}_o^2|>>|\mathcal{S}_o^1|}$ with sufficiently high $|\mathcal{S}_o^2|$. The transport is minimized by fixing $\phi_o(\mathsf{X})$ for all irrelevant parts $\mathsf{X} \in\mathcal{S}_o^2$. 
\end{proof}

\begin{figure*}[t]
    \centering
    \includegraphics[width=\textwidth]{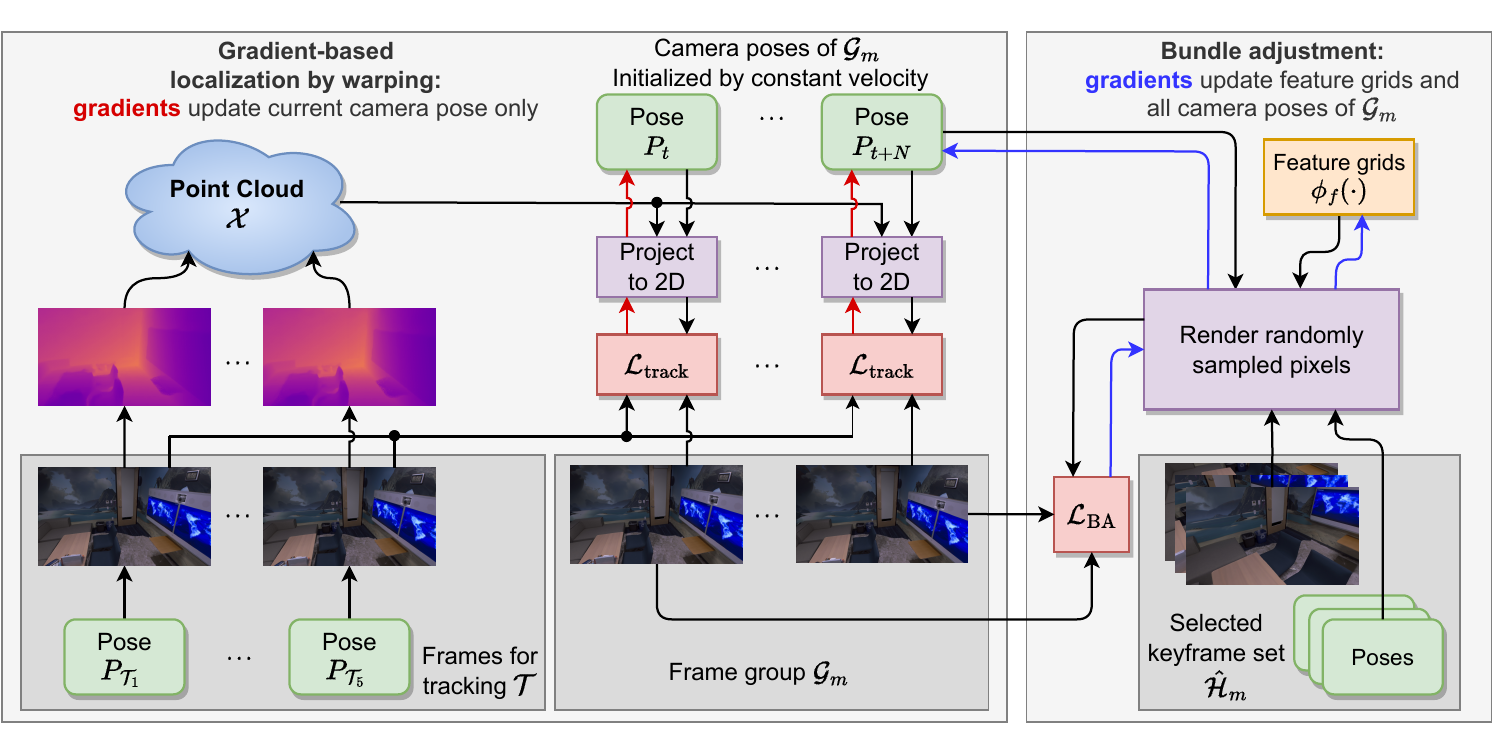}
    \caption{\textbf{Hybrid Odometry.} For simplicity $t$ refers to the index of the first frame from the current frame group $\mathcal{G}_m$ and $\mathcal{T}_1$ is the first frame from the tracking frames. In \textbf{gradient-based localization by warping}, we initialize poses by constant velocity, render the depth for some pixels on the last few frames $\mathcal{T}$ from the previous frame group to get point cloud $\mathcal{X}$, and update the camera poses by minimizing the reprojection loss of $\mathcal{X}$ to the current frame $\mathcal{G}_m$. In \textbf{bundle adjustment} we select several keyframes and together with $\mathcal{G}_m$ use VR to produce pixel values. The loss is backpropagated to the feature grids and the camera poses belonging to frames in $\mathcal{G}_m$.}
    \label{fig: hybrid_odometry}
\end{figure*}

To address Problem~\ref{pb:problem2}, a crucial unresolved matter involves dividing the set $\mathcal{S}_o$ into subsets $\mathcal{S}_o^1$ and $\mathcal{S}_o^2$. Unfortunately, this division cannot be known beforehand due to the nature of the SLAM pipeline. It is important to note that our goal is to produce binary outputs for $ \mathcal{S}_o^1$ and fixed value of $o_{init}$ for $\mathcal{S}_o^2$, in order to avoid unnecessary updates and thus make the process faster. On the other hand, imposing the prior of the necessary conditions on opacity (for binary weights) can lead to better solutions. To this end, we aim to model a softly-binarized decoder network $\phi_d(\cdot)$ during the map initialization stage. Learning this decoder jointly with $\phi_f(\cdot)$ allows us to choose $o_{init}$ meaningfully. More specifically, during the joint optimization, we set $\eta=0$ so that the outputs of $\phi_f(\cdot)$ are initialized to zero. The optimization process then automatically adjusts $\phi_d(0)$ for the scene, which also corresponds exactly to the irrelevant set $\mathcal{S}_o^2$. Note that $\phi_f(\cdot)$ values for $\mathcal{S}_o^2$ will be fixed at zeros, as they do not contribute to the optimization objective. Nevertheless, the value $\phi_d(0)$ changes in the initialization stage, after which it becomes fixed and is referred to as $o_{init}$.

\boldparagraph{Softly-Binarized Decoder.}
To achieve $\phi_d(\mathsf{X})\in\{0,1\}$ for all $\mathsf{X}\in\mathcal{S}_o^1$, we aim to binarize the decoder $\phi_d(\cdot)$ via an activation function. However, maintaining high-performance levels while binarizing activations is challenging. This challenge is due to neurons' reduced expressiveness from restricted states and discrete computational nodes, hindering gradients propagation during training~\cite{ma2019binary}. We address this by applying a softly-binarized sigmoid activation function with temperature $\tau$, tuned during the map initialization stage. Consequently, the final opacity distribution is encouraged to be ternary which offers better results in our experiments, attributed to the incorporation of real-world priors and VR formulation. We use the multi-scale grid-based representation of $\phi_f(\cdot)$, which makes the realization of our algorithm straightforward.

%% file: sec/3_3_hybrid_odometry.tex
\section{Hybrid Odometry for NeRF-SLAM}
\label{sec:SLAM}

\subsection{SLAM System Design}
\boldparagraph{Initialization.} During initialization, we take the first $N_0$ frames to train the color and opacity decoder, as well as the feature grid-based scene representation and the camera poses. For the first two frames, we use the ground truth of the camera poses. We then estimate the following $(N_0 - 2)$ camera poses under the constant velocity motion model, where each pose is derived from the previous two. The loss used in initialization is similar to the one in BA.

\boldparagraph{Hybrid Odometry.}
After initialization, the color and opacity decoders are fixed. We group the frames and perform HO through a combination of gradient-based localization (GL) and BA. Each group has $N$ frames, and thus the $m$-th group $\mathcal{G}_{m} = {\{ N_0+(m-1)N+i \}_{i=1}^N}$. For notation purposes, we define the $0$-th group, $\mathcal{G}_0$, as the frames used for initialization, consisting of the first $N_0$ frames. Note that the last group may have less than $N$ frames, depending on the total number of frames in the sequence. We perform GL for each frame in order, and at the end of the group we perform BA for the whole group. An illustration of this process is shown in Fig.~\ref{fig: hybrid_odometry}.

\begin{center}
\begin{table*}[t]
\vspace{3pt}
\caption{\textbf{Camera Tracking Results on Replica and 7-Scenes.} ATE RMSE [cm] ($\downarrow$) is used as the evaluation metric. \texttt{o-x} and \texttt{r-x} denote office-x and room-x respectively in Replica~\cite{straub2019replica}. NICE-SLAM$^{\cross[.5pt]}$ and NeRF-SLAM$^{\cross[.5pt]}$ are taken from NICER-SLAM~\cite{zhu2023nicer}, and DIM-SLAM$^*$ is the re-implementation by us. The methods in the upper section of the table utilize either GT depth or pseudo depth, whereas those in the lower section do not use any depth information.}
\begin{adjustbox}{max width=\textwidth}
  \begin{tabular}{@{}l|cccccccc|ccccccc@{}}
    \toprule
    
    \multirow{2}{*} & \multicolumn{8}{c|}{\textbf{Replica~\cite{straub2019replica}}} & \multicolumn{7}{c}{\textbf{7-Scenes~\cite{shotton2013scene}}} \\
        & \texttt{o-0} & \texttt{o-1} & \texttt{o-2} & \texttt{o-3} & \texttt{o-4} & \texttt{r-0} & \texttt{r-1} & \texttt{r-2} & \texttt{chess} & \texttt{fire} & \texttt{heads} & \texttt{kitchen} & \texttt{office} & \texttt{pumpkin} & \texttt{stairs} \\
        
    \midrule

    NICE-SLAM$^{\cross[.5pt]}$~\cite{zhu2022nice} & 0.99 & 0.90 & 1.39 & 3.97 & 3.08 & 1.69 & 2.04 & 1.55 & 2.16 & 1.63 & 7.80 & 5.73 & 19.34 & 3.31 & 4.31\\
    
    NICER-SLAM~\cite{zhu2023nicer} & 2.12 & 3.23 & 2.12 & 1.42 & 2.01 & 1.36 & 1.60 & 1.14 & 3.28 & 6.85 & 4.16 & 3.94 & 10.84 & 20.00 & 10.81 \\
    
    NeRF-SLAM$^{\cross[.5pt]}$~\cite{rosinol2022nerf} & 12.75 & 10.34 & 14.52 & 20.32 & 14.96 & 17.26 & 11.94 & 15.76 & 9.34 & 8.57 & 4.44 & 9.02 & 16.67 & 43.96 & 5.41\\
    
    \midrule
    
    DIM-SLAM$^*$ & 0.86 & \textbf{0.47} & 3.50 & 1.95 & 30.39 & 80.16 & 10.10 & 27.48 & 5.22 & 9.61 & 8.01 & 29.92 & \textbf{9.76} & \textbf{17.12} & 12.83  \\
    
    Ours & \textbf{0.59} & 1.74 & \textbf{1.70} & \textbf{0.81} & \textbf{3.47} & \textbf{4.51} & \textbf{0.91} & \textbf{7.49} & \textbf{4.86} & \textbf{6.00} & \textbf{6.30} & \textbf{7.14} & 14.14 & 18.73 & \textbf{12.40}\\
    
    \bottomrule
  \end{tabular}
\end{adjustbox}
\label{tab:tracking_results}
\end{table*}
\end{center}

\begin{table*}[t]
\caption{\textbf{Geometric (L1) and Photometric (PSNR) results on Replica.} The pair of numbers denotes Depth L1 [cm] ($\downarrow$) and PSNR [dB] ($\uparrow$), respectively. iMAP$^{\cross[.5pt]}$ and NICE-SLAM$^{\cross[.5pt]}$ are taken from  NeRF-SLAM, and DIM-SLAM$^*$ is our re-implementation.}
\begin{adjustbox}{max width=\textwidth}
  \centering
  \begin{tabular}{@{}lcccccccc@{}}
    \toprule
    
    & \texttt{o-0} & \texttt{o-1}& \texttt{o-2} & \texttt{o-3} & \texttt{o-4} & \texttt{r-0} & \texttt{r-1} & \texttt{r-2}  \\
      
    \midrule

    iMAP$^{\cross[.5pt]}$~\cite{sucar2021imap} & (6.43, 7.39) & (7.41, 11.89) & (14.23, 8.12) & (8.68, 5.62) & (6.80, 5.98) & (5.70, 5.66) & (4.93, 5.31) & (6.94, 5.64)\\

    NICE-SLAM$^{\cross[.5pt]}$~\cite{zhu2022nice} & (1.51, 22.44) & (0.93, 25.22) & (8.41, 22.79) & (10.48, 22.94) & (2.43, 24.72) & (2.53, 29.90) & (3.45, 29.12) & (2.93, 19.80)\\

    NeRF-SLAM~\cite{rosinol2022nerf} & (2.97, 34.90) & (1.98, 53.44) & (9.13, 39.30) & (10.58, 38.63) & (3.59, 39.21) & (2.97, 34.90) & (2.63, 36.95) & (2.58, 40.75)\\

    \midrule
    
    DIM-SLAM$^*$ \cite{li2023dense} & (7.47, 29.44) & (6.52, 30.48) & (18.35, 20.28) & (24.79, 19.93) & (37.45, 18.84) & (87.23, 8.97) & (23.74, 20.76) & (33.55, 20.32)\\
    
    Ours & (\textbf{3.20}, \textbf{30.34}) & (\textbf{2.15}, \textbf{31.67}) & (\textbf{10.43}, \textbf{23.05}) & (\textbf{11.76}, \textbf{23.32}) & (\textbf{13.23}, \textbf{26.19}) & (\textbf{11.86}, \textbf{21.22}) & (\textbf{4.79}, \textbf{26.29}) & (\textbf{20.95}, \textbf{22.54}) \\
    
    \bottomrule
  \end{tabular}
\end{adjustbox}
  
  \label{tab:replica_mapping}
\end{table*}

\vspace{-10pt}

\subsection{Gradient-based Localization by Warping}
We track the camera frame by frame using GL by warping. For group $\mathcal{G}_m$, we take the last five frames from the previous group (i.e., $\{ N_0 + (m-1)N + i \}_{i=-4}^0$, denoted as $\mathcal{T}$) for tracking reference, and randomly sample $P$ pixels from frames $\mathcal{T}$. For a sampled pixel $\mathsf{x}$ with RGB value $\mathbf{I}_\mathsf{x}^g$, we render its depth ${D}(\mathsf{x})$ with the estimated camera pose of its associated frame, and project it to 3D space with the depth:
\begin{equation}
    \mathsf{X} = \mathsf{R} \mathsf{K}^{-1} [\textsf{x}, 1]^{\intercal} {D}(\mathsf{x}) + \mathsf{t},
\end{equation}
where $\mathsf{R}$ and $\mathsf{t}$ are the camera poses associated to the pixel $\mathsf{x}$, $\mathsf{K}$ is the known camera intrinsic matrix. In this manner, we form a set of 3D points, say $\mathcal{X}= \{\mathsf{X}_i\}_{i=1}^P$, which will be used as a reference for tracking all frames in $\mathcal{G}_m$.

For frame $i \in \mathcal{G}_m$, we first get an initial estimation of the camera pose with the constant velocity assumption, using the estimated camera poses $[{\mathsf{R}}_{i-2}, {\mathsf{t}}_{i-2}]$ and $[{\mathsf{R}}_{i-1}, {\mathsf{t}}_{i-1}]$, to get $[{\mathsf{R}}_i, {\mathsf{t}}_i]$. Then we optimize the camera pose using GL via warping. In each optimization step, we project all points in $\mathcal{X}$ to the image space of frame $i$:
\begin{equation}
    \hat{\mathsf{x}} \sim \mathsf{K}{\mathsf{R}}_i^{\intercal} (\mathsf{X} - {\mathsf{t}}_i).
\end{equation}
All the projected 2D points form a set ${\hat{\mathcal{X}}=\{\hat{\mathsf{x}}|\mathsf{X}\in\mathcal{X}\}}$, in which we filter out the points outside of the image boundaries to get $\hat{\mathcal{X}}_b \subseteq \hat{\mathcal{X}}$.

Finally, $L_1$ loss can be calculated between the RGB value $\mathbf{I}_\mathsf{x}^g$ at original location $\mathsf{x}$ and bilinearly interpolated RGB value at projected location ${\hat{\mathsf{x}}}$ on current frame $i$, denoted as $\textbf{I}_{\hat{\mathsf{x}}}^i(\cdot)$, which is a function of camera pose $[{\mathsf{R}}_i, {\mathsf{t}}_i]$. Considering all points in $\hat{\mathcal{X}}_b$, we define the tracking loss as:
\begin{equation}
    \mathcal{L}_{\text{track}} = \sum_{\hat{\textsf{x}} \in \hat{\mathcal{X}}_b} {\lVert \textbf{I}_{\textsf{x}}^g - \textbf{I}_{\hat{\mathsf{x}}}^i(\mathsf{R}_i,\mathsf{t}_i) \rVert}_1,
\end{equation}
Note that minimizing $\mathcal{L}_\text{track}$ improves the camera pose parameters. During this process, the functions $\phi_o(\cdot)$,$\phi_c(\cdot)$, used during the volumetric process, are not involved. This omission leads to a significant increase in overall speed, as the VR process is computationally expensive.

\subsection{Bundle Adjustment}

After optimizing the estimated camera poses of all frames in group $\mathcal{G}_m$ through GL, we conduct BA to jointly optimize these poses and the neural scene representations.

\boldparagraph{Global Keyframe Set.}
We maintain a global keyframe set to ensure consistent scene reconstruction, which is defined for frame group $\mathcal{G}_m$ as
\begin{equation}
    \mathcal{H}_m = \{1, 1+H, 1+2H, ... \} \subseteq \bigcup_{i=0}^{m-1} \mathcal{G}_i,
\end{equation}
where $H$ sets the keyframe interval. For the computational efficiency, we select a subset $\hat{\mathcal{H}}_m$ based on the overlapping ratio $r$, which measures reprojected pixel alignment between a keyframe $h \in \mathcal{H}_m$ and the last frame in $\mathcal{G}_m$. Only frames with $r \geq R$ will be randomly selected to reach a total number of $L$, where $R$ is a predefined threshold.

\boldparagraph{Optimization.}
The whole frame set used for BA is $\mathcal{B}_m = \mathcal{G}_m \cup \hat{\mathcal{H}}_m$. We optimize camera poses in $\mathcal{G}_m$ and feature grids in BA, and fix the camera poses in $\hat{\mathcal{H}}_m$. 

For frame $i \in \mathcal{B}_m$, we randomly sample $Q$ pixels to get pixel set $\mathcal{Q}_i$. The photometric rendering loss is: 
\begin{equation}
\label{eq:ba_rgb}
    \mathcal{L}_{\text{rgb}} = 
    \sum_{i \in \mathcal{B}_m} 
    \sum_{\mathsf{p} \in \mathcal{Q}_i} 
    \lVert 
    \mathbf{I}_{\mathsf{p}} 
    - 
    \hat{\mathbf{I}}_{\mathsf{p}} \left( \mathsf{R}_i,\mathsf{t}_i, \mathcal{F} \right) 
    {\rVert}_1,
\end{equation}
where $\mathbf{I}_{\mathsf{p}}$ and $\hat{\mathbf{I}}_{\mathsf{p}}(\cdot)$ are the ground truth and rendered RGB values at pixel $\mathsf{p}$, $[{\mathsf{R}}_i, {\mathsf{t}}_i]$ represents the camera pose, and the set of features $\mathcal{F}$ represents $\phi_f(\cdot)$.

Following DIM-SLAM \cite{li2023dense}, we use patch-based warping for geometric consistency. For each pixel $\mathsf{p} \in \mathcal{Q}_i$ in frame $i$, we generate patches $\mathsf{P}_z$ of size $z \times z$, with $z$ from set $Z = \{1, 7, 11\}$, and project them to 3D space and reproject back to frames in $\mathcal{B}_m$ excluding frame $i$. Patches with less than 5 valid reprojections are removed. The patch-based warping loss is defined as:
\begin{equation}
\label{eq:ba_warping}
    \mathcal{L}_{\text{warping}} = \sum_{i \in \mathcal{B}_m} \sum_{\mathsf{p} \in \mathcal{Q}_i} \sum_{z \in Z} \alpha_z \mathcal{L}_{\text{SSIM}} (\mathsf{P}_z, i),
\end{equation}
where $\alpha_z$ is the loss scalar for patch size $z$. The structure similarity loss \cite{wang2004image} function $\mathcal{L}_{\text{SSIM}}$ takes a patch and its associated frame id as input, defined as:
\begin{equation}
\label{eq:ba_ssim}
    \mathcal{L}_{\text{SSIM}} (\mathsf{P}, i) = \sum_{j \in \mathcal{B}_m, j \neq i} \text{SSIM}(\mathbf{I}_{\mathsf{P}}, \hat{\mathbf{I}}_{\mathsf{P}}(\mathsf{R}_i,\mathsf{t}_i, \mathsf{R}_j,\mathsf{t}_j, \mathcal{F})),
\end{equation}
where $\mathbf{I}_{\mathsf{P}}$ is the GT RGB values for patch $\mathsf{P}$, and $\hat{\mathbf{I}}_{\mathsf{P}}$ is the reprojected patch $\mathsf{P}$ from frame $i$ to frame $j$. 

Finally, the total loss for BA is the scaled sum:
\begin{equation}
    \mathcal{L}_{\text{BA}} = \alpha_{\text{rgb}} \mathcal{L}_{\text{rgb}} + \alpha_{\text{warping}} \mathcal{L}_{\text{warping}},
\end{equation}
where $\alpha_{\text{rgb}}$ and $\alpha_{\text{warping}}$ are scalars.

%% file: sec/4_experiments.tex
\section{Experiments}
\label{sec:experiments}

\begin{table}[!t]
\vspace{3pt}
\caption{\textbf{Ablation Study.} `TT' and `HO' denote ternary-type opacity and hybrid odometry respectively. Results obtained on Replica~\cite{straub2019replica} \texttt{office-0}.}
\begin{adjustbox}{max width=\linewidth}
  \centering
  \begin{tabular}{@{}lccc@{}}
    \toprule
    
    & ATE RMSE [cm] $\downarrow$ & Depth L1 [cm] $\downarrow$ & PSNR [dB] $\uparrow$ \\
      
    \midrule

    w/o TT, HO & 70.00 & 47.95 & 16.57 \\
    w/o HO & 40.95 & 55.62 & 19.59\\
    w/o TT & 2.11 & 5.20 & 29.27\\
    Ours & \textbf{0.59} & \textbf{3.20} & \textbf{30.34} \\
    
    \bottomrule
  \end{tabular}
\end{adjustbox}

  \label{tab:abalation}
\end{table}

\subsection{Implementation Details}

We employ scene feature grids across 7 hierarchies, each with three color and one opacity channel, decoded by two MLPs with three hidden layers, where the initial layers have 32 dimensions with ReLU activation, and the final employs Sigmoid activation with temperature parameters $\tau$. For SLAM initialization, we use 15 frames, training over 1500 iterations, initially applying L1 loss for depth, then incorporating $\mathcal{L}_{\text{warping}}$ and finally $\mathcal{L}_{\text{rgb}}$, adjusting camera poses after 150 iterations. The challenging \texttt{kitchen} scene from 7-Scenes necessitates a group size $N=1$, unlike others set at $N=10$. For GL, we sample $P=10000$ pixels in the tracking frame set $\mathcal{T}$ and optimize the camera poses for 200 iterations. For BA, we take the global keyframe frequency $H=5$, we select $L=10$ frames from the global keyframe set with the overlapping threshold $R=0.1$, and run for 300 iterations. Our system, tested on an NVIDIA A100 GPU, averages 6ms per iteration for Global Localization (GL) and 145ms for Bundle Adjustment (BA), with learning rates tailored for different stages and components, and optimal sigmoid temperatures for opacity and RGB found to be $\tau_1=10$ and $\tau_2=10$ after extensive testing.

\subsection{Datasets}

\label{subsection:dataset}
We evaluate our method on synthetic dataset Replica~\cite{straub2019replica} and real-world dataset 7-Scenes~\cite{shotton2013scene}. We use 7 layers of multi-resolution feature grids, similar to \cite{li2023dense}. For Replica, we set the multi-resolution voxel sizes to $\{0.64, 0.48, 0.32, 0.24, 0.16, 0.12, 0.08 \}$m, and we sample $Q=3000$ pixels in BA. For more challenging dataset 7-Scenes, we set a finer multi-resolution feature grids with sizes $\{0.48, 0.32, 0.24, 0.16, 0.12, 0.08, 0.04 \}$m to capture more details. We sample $Q=5000$ pixels in BA, as using more pixels yields marginal accuracy improvements at an increasing computational cost.

\subsection{Metrics}
We assess tracking accuracy using the RMSE of Absolute Trajectory Error (ATE RMSE), after aligning and scaling trajectories to ground truth by ~\cite{grupp2017evo}. Mapping quality is evaluated on Replica \cite{straub2019replica} dataset for both geometric and photometric aspects. In line with NeRF-SLAM \cite{rosinol2022nerf}'s approach, we render depth and RGB for all frames using estimated camera poses post-training. For the geometric evaluation, we rescale the estimated depth, filter out outlier pixels with a depth beyond 1m compared from GT, and calculate mean L1 loss. For the photometric evaluation, we measure PSNR between rendered and GT RGB images. The final reported L1 loss and PSNR are averaged across all frames.

\begin{figure}[t]
  \centering
  \vspace{5pt}
   \includegraphics[width=\linewidth]{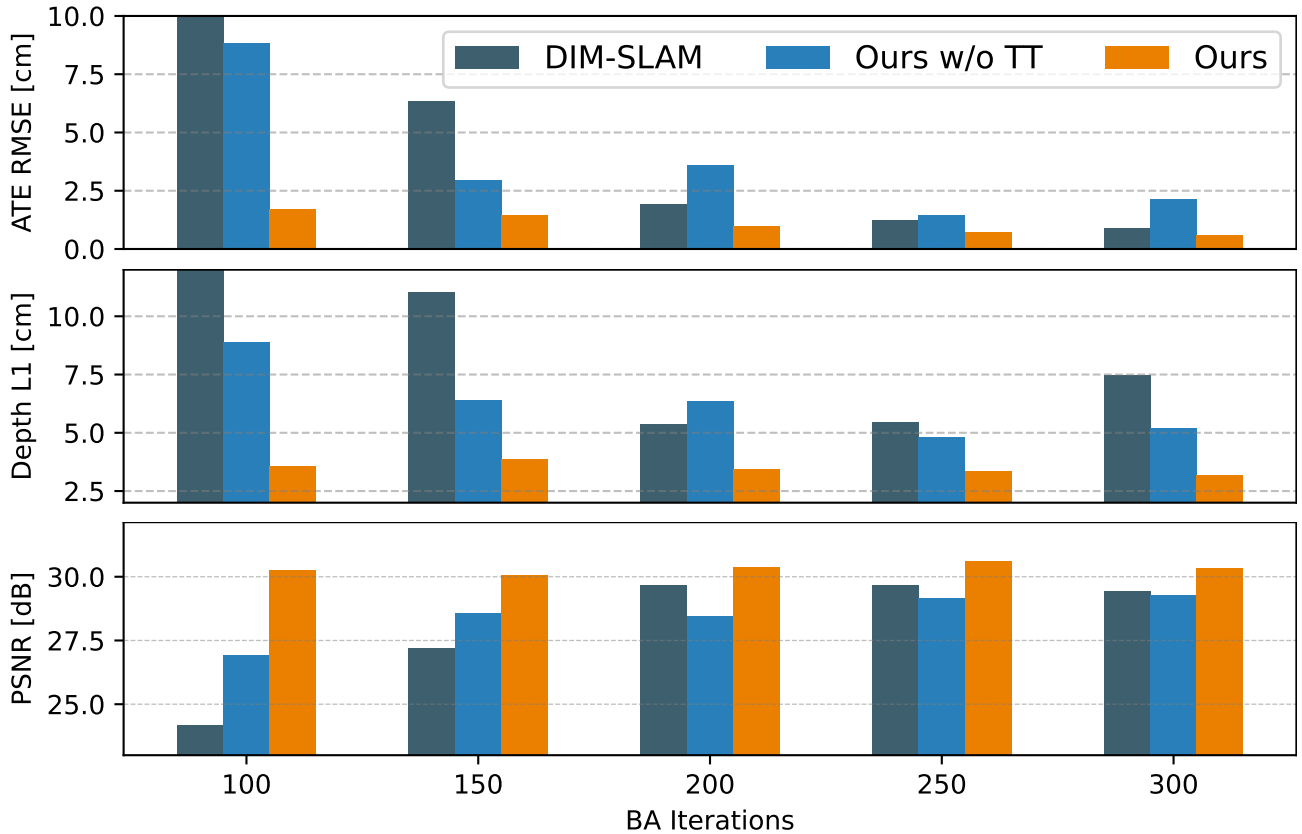}
   \caption{\textbf{Ablation Study.} 
   Our evaluation on Replica~\cite{straub2019replica} \texttt{office-0} with varying BA iterations shows that, unlike DIM-SLAM~\cite{li2023dense} which struggles with fewer BA iterations, our method maintains performance even with reduced iterations. This suggests the ternary-type (TT) opacity contributes to map training convergence speed and system robustness. 
   }
   \label{fig:ablation_study_sigmoid_expression}
\end{figure}

\subsection{Quantitative and Qualitative Results}
We present both quantitative and qualitative comparisons of our method against others. For the DIM-SLAM baseline, we re-implement the code and standardize the hyper-parameters to ours. As shown in TABLE \ref{tab:tracking_results}, our method surpasses DIM-SLAM in tracking across most scenes in Replica~\cite{straub2019replica} and 7-Scenes~\cite{shotton2013scene} datasets. In TABLE \ref{tab:replica_mapping}, our method excels in all scenes in both color and depth mapping. Visual evidence in Fig. \ref{fig:teaser} (left) demonstrates accurate rendered RGB image and depth map with minimum artifacts. Furthermore, Fig. \ref{fig:teaser} (right) shows that our method is approximately 6x faster than DIM-SLAM, while achieving lower tracking and mapping errors.

\subsection{Ablation Study}

\boldparagraph{Effectiveness of Ternary-Type Opacity.} TABLE \ref{tab:abalation} showcases the enhanced performance in both mapping and tracking with TT. Furthermore, in Fig. \ref{fig:ablation_study_sigmoid_expression}, we illustrate that our TT enables a reduction in the number of iterations in BA, with a minimal decrease in performance compared to DIM-SLAM. Furthermore, at the end of the training, we randomly sample a ray and plot the decoded opacity and calculated weights along the ray, comparing scenarios with or without TT, as shown in Fig. \ref{fig:along_ray}. With TT, the weights are more concentrated near the depth with a higher peak. In contrast, without TT, the weights are less concentrated with an apparent non-zero region before the depth, and have a lower peak. The right part of Fig. \ref{fig:along_ray} further illustrates that achieving low weights before the depth and high weights around the depth is facilitated by maintaining 0-1 opacity along the ray.

\boldparagraph{Effectiveness of Hybrid Odometry.} With our TT, we compare the tracking accuracy with and without using HO, specifically employing camera poses derived from the constant velocity model without further optimization. As indicated in TABLE \ref{tab:abalation}, using HO is crucial for achieving favorable mapping and tracking results.

%% file: sec/5_conclusion.tex
\section{Conclusion}
\label{sec:conclusion}
We present a method for RGB-only NeRF-SLAM that benefits from the prior of the opaque scenes, achieved by the TT modelling of the 3D scenes.
Furthermore, we propose a hybrid method to estimate the camera motion that leads to a significant increase in overall speed. The theoretical insights that we provide in analyzing the VR and opaque surfaces, in our context, are well supported by our experimental results. In fact, the reported observation has led us to propose a simple yet very effective strategy to exploit the opaque surface prior, which in turn offers us  both improved accuracy and speed, thanks to the faster convergence offered by the proposed TT prior.

\boldparagraph{Limitations and Future Work.} While being online, the proposed method is not yet real-time on the consumer devices for many common applications. These requirements may be addressed by application- and hardware-specific code optimization and the system configurations, which remains as the future works.